\documentclass[11pt]{article}
\usepackage{fullpage}
\usepackage{booktabs} %
\usepackage{amsmath,amssymb,amsfonts}

\usepackage{amsthm}

\usepackage{natbib}
\usepackage{hyperref}
\usepackage[T1]{fontenc}

\newcommand{\numitems}{N}

\newcommand{\R}{\mathbb{R}}
\newcommand{\E}{\mathbb{E}}
\newcommand{\latentdim}{d}
\newcommand{\diag}{\mathrm{diag}}

\newcommand{\sgn}{\mathrm{sgn}}

\newtheorem{thm}{Theorem}
\newtheorem{coro}[thm]{Corollary}
\newtheorem{prop}[thm]{Proposition}

\newtheorem{remark}{Remark}
\theoremstyle{definition}

\theoremstyle{definition}
\newtheorem{definition}{Definition}

\newcommand{\timeoffset}{s}
\newcommand{\affinity}{\mathsf{a}}
\newcommand{\stationary}{\mathsf{s}} \usepackage{xcolor}
\usepackage{algorithm}
\usepackage{algpseudocode}
\usepackage{graphicx}

\newtheorem{lemma}[thm]{Lemma}

\title{Preference Dynamics Under Personalized Recommendations}

\author{Sarah Dean$^1$ and Jamie Morgenstern$^{2,3}$\\
$^1$Cornell University~~~~$^2$The University of Washington~~~~$^3$Google
}

\begin{document}

\maketitle

\begin{abstract}
The design of content recommendation systems underpins many online platforms: social media feeds, online news aggregators, and audio/video hosting websites all choose how best to organize an enormous amount of content for users to consume. Many projects (both practical and academic) have designed algorithms to match users to content they will enjoy under the assumption that user's preferences and opinions do not change with the content they see. 

However, increasing amounts of evidence suggest that individuals' preferences are directly shaped by what content they see---radicalization, rabbit holes, polarization, and boredom are all example phenomena of preferences affected by content. Polarization in particular can occur even in ecosystems with ``mass media,'' where no personalization takes place, as recently explored in a natural model of preference dynamics by~\citet{hkazla2019geometric} and~\citet{gaitonde2021polarization}. If all users' preferences are drawn towards content they already like, or are repelled from content they already dislike, uniform consumption of media leads to a population of heterogeneous preferences converging towards only two poles. 

In this work, we explore whether some phenomenon akin to polarization occurs when users receive \emph{personalized} content recommendations. We use a similar model of preference dynamics, where an individual's preferences move towards content the consume and enjoy, and away from content they consume and dislike. We show that standard user reward maximization is an almost trivial goal in such an environment (a large class of simple algorithms will achieve only constant regret). A more interesting objective, then, is to understand under what conditions a recommendation algorithm can ensure stationarity of user's preferences. We show how to design a content recommendations which can achieve approximate stationarity, under mild conditions on the set of available content, when a user's preferences are known, and how one can learn enough about a user's preferences to implement such a strategy even when user preferences are initially unknown.

\end{abstract}

 \section{Introduction}
Content recommendation pervades the day-to-day experience of nearly every person with a smartphone.
An increasing proportion of the news we consume, the advertisements we see, the entertainment we watch, and the products we buy are served to us by online platforms. Some of these systems provide the same information and recommendations to all viewers (e.g., many newspapers have a conventional ``front page" on their homepage), while others make personalized suggestions, each of which depend on information a platform has about their users. 

Personalized recommendations may have several objectives: they may aim to serve individuals with content they will like best (user utility) or to increase the amount of time or money an individual will spend on the platform (on-time). Both personalized and non-personalized content may be designed for these or other goals, and in the case of news, opinion pieces, and advertisements, may aim to  inform or persuade individuals of something (persuasion). 
The latter persuasion objective explicitly accounts for individuals' opinions or preferences changing due to the content they see, while user utility and on-time objectives generally encode preferences as fixed. However, even if these objectives capture the platform's goals, mounting evidence suggests that individuals' preferences are not static, but change as they interact with platforms. Examples of phenomena which support this hypothesis include polarization of opinions; individuals entering into rabbit holes, or being radicalized by online content; and boredom with content which is too similar to previous content.

Motivated by these effects of content providers on user beliefs or preferences, recent work has aimed to understand in what contexts non-personalized content recommendation can lead to polarization~\citep{hkazla2019geometric,gaitonde2021polarization}. The models capture \emph{biased assimilation}, assuming individuals whose beliefs initially align with content they observe will update their beliefs to align more closely with that content, while individuals who disagree with content they see will disagree more strongly after seeing that content.
This work finds, in a variety of conditions, a set of individuals who consume the same content will be drawn to one of two poles, regardless of their initial preferences or opinions. 
As such, non-personalized content platforms can drive a group towards polarized opinions even if their initial opinions were not polarized.

In this work, we aim to understand preference dynamics in personalized content recommendation systems, 
where different individuals are exposed to different content.
Understanding the consequences of preference changes is crucial both for effectively optimizing any long-term objective and for designing systems whose trajectory can be controlled to guarantee certain external desiderata. 
We study the extent to which personalized recommendations lead to analogues of polarization, how these dynamics interact with the goal of maximizing user utility, and perhaps more importantly how to make recommendations which avoid changing or manipulating the preferences of individuals. 
\subsection{Our Contributions}
In this work, we introduce a model of preference dynamics for personalized recommendation with biased assimilation. 
We study both the preference trajectories induced by recommendations and the performance of recommendation policies, as measured by regret.
For this model, we show
\begin{itemize}

\item {\bf The problem of regret minimization is nearly trivial for a single user}: one need only find the hemisphere in which that individual's initial preferences reside, then repeatedly recommend any piece of content in that hemisphere. This leads to constant regret.
\item Consequently, a more technically interesting question is {\bf whether one can design a recommendation system which does not cause preferences of individuals to change too much.} This also has practical importance: can we serve personalized content to users that they will enjoy, and for which we are not (intentionally or unintentionally) swaying their preferences? 
\item Given knowledge of an individual's initial preferences, {\bf we develop an algorithm that computes a distribution over content which leaves the preferences approximately unchanged}
so long as the available content set $Q$ satisfies a technical notion of richness. 
We show that this algorithm achieves regret logarithmic in the time horizon when performance is defined in terms of closeness to the individual's initial preferences.
\item
{\bf We present conditions under which an individual's initial preferences can be identified from observed reactions to recommendations.} This extension is both of practical and technical interest---most content recommendations learn about user's preferences based on their interaction with content. Furthermore, the identification problem must contend with changing preferences, which are caused by the the very act of observation (i.e., recommending a piece of content).

\end{itemize}

\section{Related Work}

\subsection{Opinion and Preference Dynamics}
The model of biased assimilation that we study was first proposed by~\citet{hkazla2019geometric} in the context of opinion polarization.
In contrast to our work, they focus on a non-personalized setting, where every individual is exposed to the same content at each time step.
They show that even when content is chosen uniformly at random, the biased assimilation dynamics result in a polarized population almost surely on an infinite horizon.
\citet{gaitonde2021polarization} consider a broader class of opinion dynamics in this non-personalized setting and develop a more nuanced analysis of the polarization phenomena on infinite time horizons.

There is also a rich body of work in economics which models how preferences and opinions change.
Often, the updates are implicitly defined by a maximization, resulting in a fixed-point update rule instead of than the iterative dynamics we study here.
Rather than attempt a complete summary, we refer to recent work by~\citet{bernheim2021theory} and citations therein.

\subsection{Dynamics in Personalized Recommendations}
Nonstationarity in recommendation systems (and online communities more generally) has received both theoretical and empirical study. 
Theoretical models of the feedback between recommendations and preferences have studied in the context of 
echo chambers and filter bubbles~\citep{jiang2019degenerate,kalimeris2021preference}.
Within a social network, models of opinion aggregation can lead to polarization~\cite{dandekar2013polarization}.
These phenomena are examples of how beliefs might be made more extreme within a recommendation system or social network due to positive feedback loops.

Evidence also suggests 
that users receiving highly personalized (or highly stationary) recommendations bore quickly of the highly similar content~\citep{wang2003personalization}, indicating some degree of preference shift as a function of previous recommendations. 
In contrast to the model we study in this work, the dynamics of boredom induce a negative feedback loop.
Some work has attempted to develop recommendation systems which reduce the negative performance effects of boredom~\citep{warlop2018fighting,leqi2021rebounding}.

There is also a large body of work
studying the dynamical effects of personalized recommendations in simulation, where more complex models of behavior can be investigated. 
If a recommender does not account for its sampling bias, this can lead to feedback loops and inconsistent naive estimators~\citep{schmit2018interaction}.  
When users are modelled as having limited knowledge of their utility, this can result in homogenization of recommendations~\citep{chaney2018algorithmic}. 
Further negative effects from feedback loops include popularity bias, echo chambers, and filter bubbles ~\cite{mansoury2020feedback,jiang2019degenerate}.
In response to this variety of simulation studies, several simulation frameworks have been proposed, variously focusing on the societal consequences of recommendation systems~\citep{lucherini2021t}, disentangling the effects of recommendations vs. user behaviors
\cite{yao2021measuring}, and methodologies for performance evaluation
\cite{krauth2020offline}.

\subsection{Machine Learning and Manipulation}

The issue of feedback between learned models and the populations on which they are trained 
has received recent attention
under the name \emph{performative prediction}, introduced by
\citet{perdomo2020performative}.
This framework 
defines an abstraction capturing user reaction via decision-dependent distribution shifts.
Much of this literature focuses on a stateless setting, with shifts occuring in a memoryless manner, through recent extensions consider stateful distributions \citep{ray2022decision,brown2022performative}. 
The definition of performative distribution shift
abstracts  away
the relationship between predictions and actions.
In contrast, our work explicitly considers how different strategies for selecting recommendations based on estimated preferences lead to different behaviors.

Reinforcement learning (RL) explicitly contends with the challenges of choosing actions in dynamical environments. It has been used in the setting of recommendation; for a broad survey on this approach see~\citet{afsar2021reinforcement}. A reinforcement learning problem is defined by a set of states $S$, actions $A$, and a reward function $r : A \times S \to \mathbb{R}$ and transition function $g : A \times S \to \Delta^S$, defining the reward and distribution over states a system will receive when starting in state $s$ and taking action $a$. To encode a recommendation system into this framework, the state $s$ is identified with the current user and their history of interaction with the system, actions $a$ as the item(s) recommended, the reward function as the feedback from the user about the recommendation, and the transitions as the effect of the recommendation on the state. 
In principle, RL algorithms can find high reward policies (rules for choosing an action based on the state)  without knowing the reward and transition functions a priori.
In the setting of recommendation systems, RL could contend with a variety of dynamics around preference shifts and boredom~\citep{ie2019reinforcement,chen2019top},
even handling the second order effects of recommendations on content creators~\citep{zhan2021towards}.

However, the use of RL for recommendation brings to bear questions of
value alignment: what do we really want to optimize~\cite{stray2021you}?
Furthermore, what are the implications of exploiting human behaviors toward the maximization of some reward signal?
A recent line of work criticizes RL-based recommendation algorithms which induce preference shifts, arguing that goals of non-manipulation should be explicitly accounted for in algorithm design~\cite{carroll2021estimating,krueger2020hidden,evans2021user}.
This perspective motivates our investigation into the objective of ``stationary preferences'' in the latter half of this paper.

\section{Setting}

We consider the effects of recommendations on the preferences of individuals. 
For consistency with the vocabulary of the recommendation systems literature, we will refer to the objects being recommended as ``items.''
An individual (sometimes called a ``user'') is represented by their preferences $p_t\in\R^d$ at time $t$. 
This preference vector could represent affinity towards various attributes of items.
We consider a finite collection of $N$ items
where each item $i$ is also represented by a vector $q_i\in\R^d$
and the set of items is denoted $ Q = \{q_1,\dots, q_N\}$.
The vector representation of items can correspond to the extent to which a particular item has various attributes. 
Then the overall affinity of an individual towards some item is modelled as the inner product of these vectors: $p_t^\top q_i$.
For example, if items are movies, then each dimension of $q$ could represent a genre, theme, or other relevant movie feature. 
Correspondingly, each dimension of $p$ represents the affinity of an individual towards movies with particular themes or genres.

Such inner product models are canonical in field of recommender systems, and by some measures remain state-of-the-art~\cite{rendle2019difficulty,dacrema2021troubling}. 
When an item represented by $q$ is consumed by an individual represented by $p$, the overall affinity $p^\top q$ is usually assumed to be observed through explicit actions like ratings or implicit behaviors like watchtime.
As such, we will sometimes refer to this inner product as the ``rating'' or ``affinity.''
In content based models, item features $q$ are assumed to be measured directly. Then the user preferences $p$ can be learned from data via linear regression.
In matrix factorization models, when neither item features nor user preferences are directly observed, ratings from a large number of users and items are used to learn a low rank model which recovers such vectors.
These methods assume that the latent representations are static and do not account for dynamic changes.

In this work, we consider dynamic updates to user preferences.
Inspired by related work on opinion polarization~\cite{hkazla2019geometric,gaitonde2021polarization}, we consider the
dynamics model capturing  notion of \emph{biased assimilation}:
when a piece of content $q_t\in Q$ is recommended at time $t$, the individuals preference vector updates as
\[\tilde p_{t+1} = p_t + \eta_t p_t^\top q_t \cdot q_t,\quad
p_t = \frac{\tilde p_t}{\|\tilde p_t\|_2}\]
for $t=0,1,\dots$ where $\eta_t$ is a \emph{step size} further discussed below.
For simplicity of exposition, we model preference and item vectors as unit norm, so that for all $t$, $p_t\in\mathcal S^{d-1}$ and $Q\subset \mathcal S^{d-1}$.
The direction of these vectors, rather than their norm, impacts affinity and dynamics.

Recalling that the inner product represents the affinity that an individual has towards an item, 
the connection to biased assimilation becomes clear.
If an individual likes an item ($p^\top q>0$), then their preference vector will move towards it proportional to their affinity. 
If an individual doesn't like an item ($p^\top q<0$), then their preference vector moves away.
Therefore, the affinity is more than a signal about the static preferences of an individual.
Rather, it is also a mechanism by which the preferences evolve.
\citet{carroll2021estimating} propose similar dynamics in the setting of recommendation for simulation experiments.

The step size $\eta_t$ controls the magnitude of the preference evolution.
We will consider two settings: fixed step size and decreasing step size.
In the fixed case, we have that $\eta_t=\eta$ for some constant $\eta>0$.
This is the setting studied by~\citet{gaitonde2021polarization,hkazla2019geometric}.
In the decreasing case, we have that $\eta_t=\frac{\eta}{t+\timeoffset}$ for an integer value of $\eta$ and $\timeoffset$.
This setting encodes the idea that the sensitivity of individuals to recommendations decreases over time.
The particular $\frac{1}{t}$ scaling is related to averaging over consumed items.

In this work we take the perspective of a recommendation algorithm designer, and investigate methods for selecting some item $q_t$ at each time step $t$.
We focus on achieving high performance for two distinct notions of performance.
For the first, the reward at each time is given by
\[r_t^\affinity = p_t^\top q_t\]
i.e. the affinity of the individual for the recommended item.
This is a classical objective for recommendations, corresponding to maximizing predicted rating or watchtime. We refer to this setting as \emph{affinity maximization}.
For the second notion, the reward is instead given by the alignment of the individual's current vector with their initial preferences:
\[r_t^\stationary = p_t^\top p_0\:.\]
This encodes the importance of \emph{stationary preferences} or non-manipulation. 
In both cases, we will present recommendation algorithms which recieve high reward, as compared with the maximum possible reward $r_\star = 1$ (which would, in both cases, be achieved by $q_t=p_0$ for all $t$).
We  evaluate the performance of policies by their regret:
\[R(T) = \sum_{t=0}^{T-1} r_\star - r_t\]
We further analyze convergence properties of the preference trajectories under such policies.
In what follows, we assume that the set of items to be recommended $Q$ and their representations $q_1,\dots,q_N$ are known and fixed. 
For the individual's preference vector, we consider both
the full observation case where $p_t$ is observed directly and the partial observation case where only the affinities $y_t = q_t^\top p_t$, possibly with noise, can be observed.

\section{Rating Maximization and Convergence}\label{sec:conv}

The dynamics we have introduced posit that an individual's preferences move towards recommended items that they have positive affinity for.
This is a self-reinforcing feedback loop--an individual's affinity for the recommended item only increases after consumption. For this reason,  rating maximization seems almost trivial 
to optimize. 
In this section, we show this is indeed the case.
In particular, any policy that recommends any one item 
in the same hemisphere as the individual's preferences 
a $1-\rho$ fraction of rounds will incur at most $\rho$ regret. 

\begin{figure}
    \centering
    \includegraphics[width=0.49\textwidth]{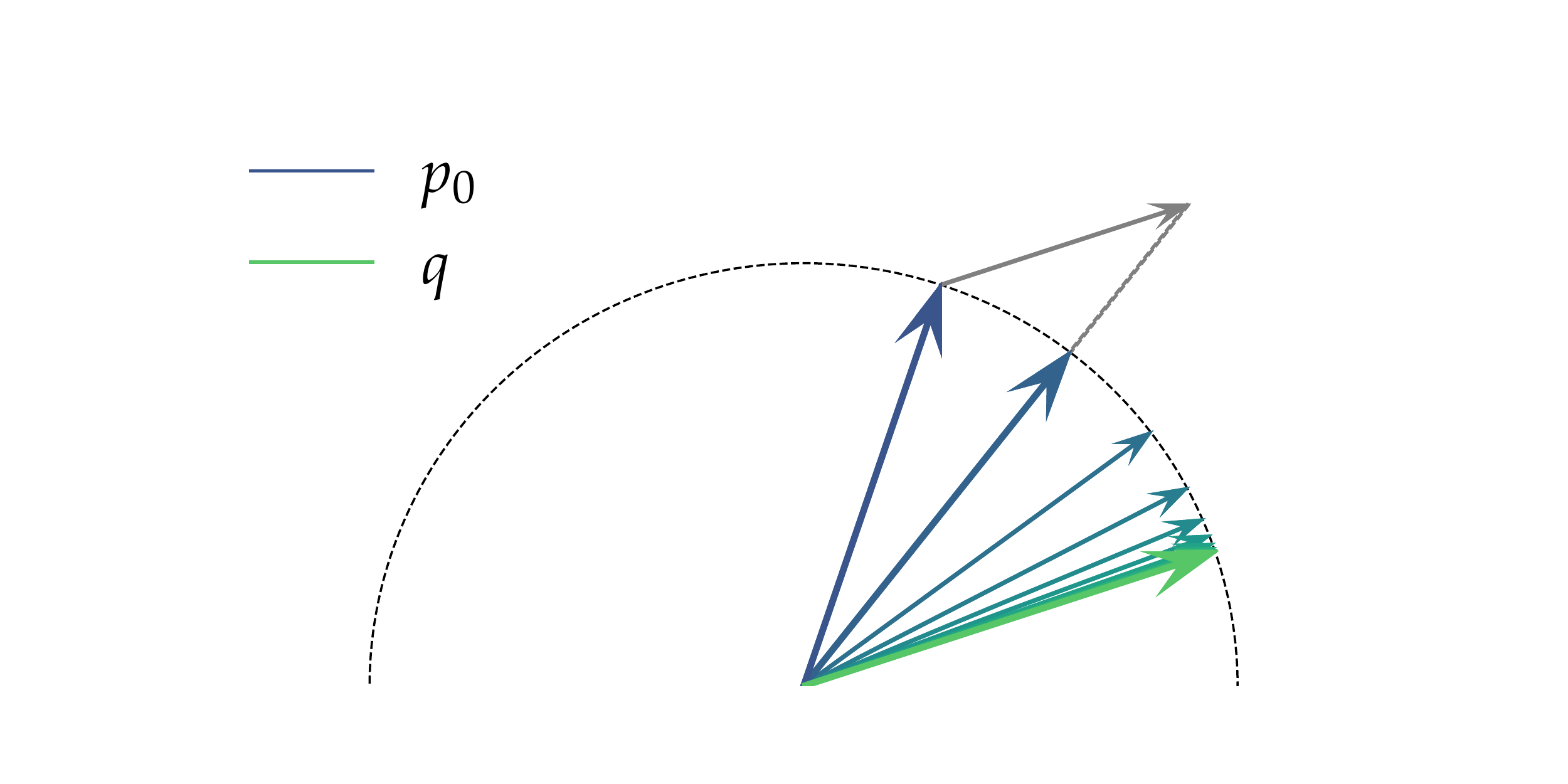}~\includegraphics[width=0.49\textwidth]{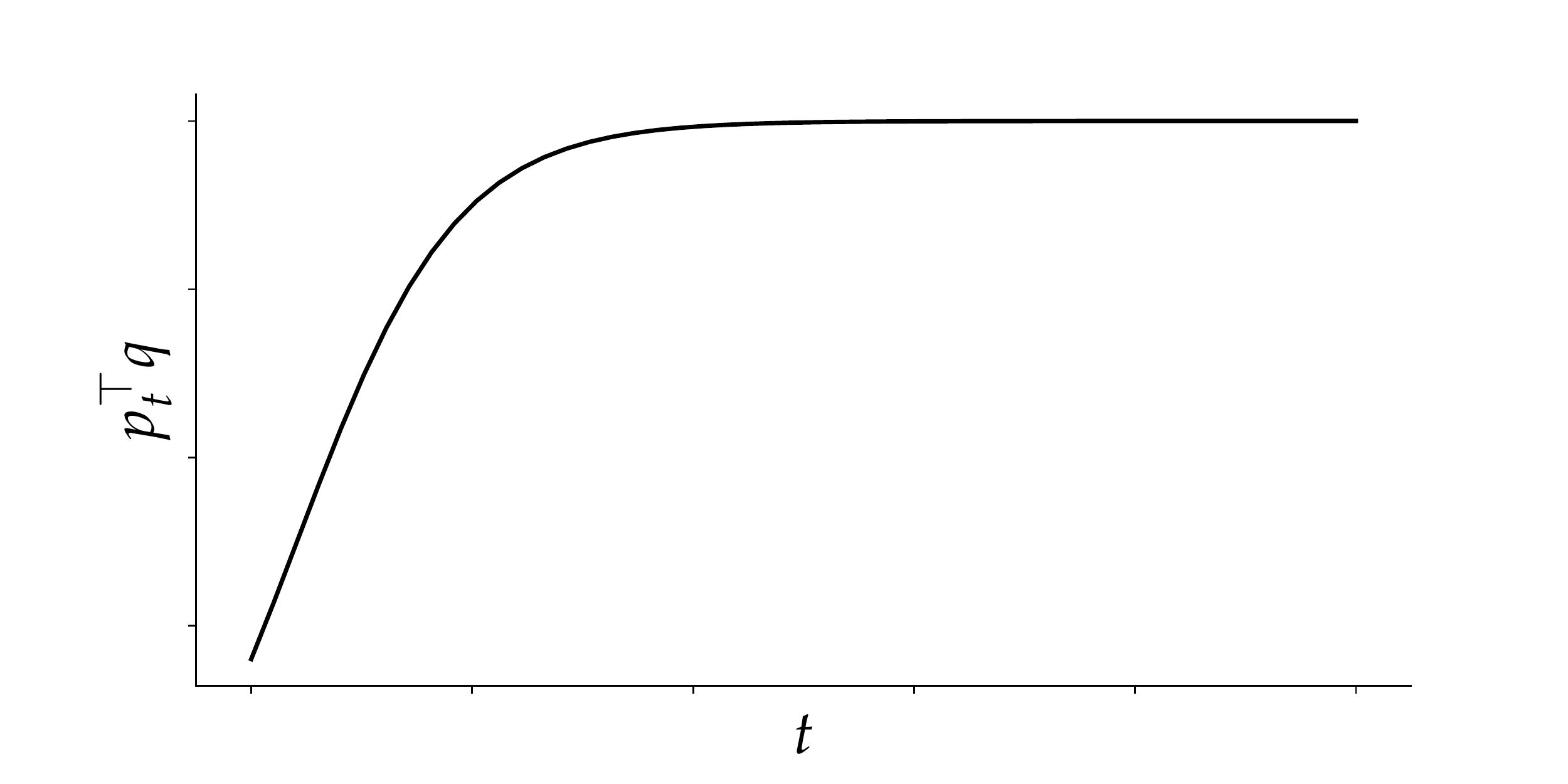}
    \caption{Left: an illustration of preference evolution from $p_0$ in blue towards the fixed recommendation $q$ in green. The first update $\eta\cdot p_0^\top q\cdot q$ is shown in grey. Right: a plot of the affinity over time.}
    \label{fig:fixed}
\end{figure}
\subsection{Convergence of Fixed Recommendations}

Consider the fixed recommendations $q_t = q$ for some fixed $q\in Q$ and all $t\geq 0$.
The following proposition characterizes the preference trajectory under this policy, showing that
$p_t$ converges to either $q$ or $-q$ depending on the sign of $q^\top p_0$.
The convergence rate is stated in terms of $(p_t^\top q)^2$, which can be understood as the squared cosine of the angle between $p_t$ and $q$.

\begin{prop}\label{prop:convergence}
If $q_t=q$ for all $t\geq 0$, then user preference $p_t$ converges to $q$ if $p_0^\top q>0$ and $-q$ if $p_0^\top q<0$.
Furthermore, the trajectory $(p_t)_{t\geq 0}$ lies in the cone defined by $p_0$ and $\pm q$.
In particular, trajectory is characterized by:
\begin{align} \label{eq:convergence}
    (p_t^\top q)^{-2}= 1+ \gamma_t^2  \left(( p_0^\top q)^{-2}- 1\right) ,\quad \gamma_t^2= \begin{cases}  (\eta  + 1)^{-t}& \eta_t=\eta\\ \prod_{k=0}^{\eta-1}
    \frac{\timeoffset+k}{t+\timeoffset+k}
     & \eta_t= \frac{\eta}{t+\timeoffset}\end{cases}
\end{align}
\end{prop}

As $t\to\infty$,  $\gamma_t$ tends to $0$, leading to $(p_t^\top q)^2$ tending to $1$.
For both step sizes settings, the convergence is fast: exponential or inverse squared. 
For the decreasing step size with $\eta=1$ and $s=1$, the expression simplifies to $\gamma_t^2=(t+1)^{-2}$.
Figure~\ref{fig:fixed} illustrates the fast convergence for fixed stepsize.

\begin{proof}
We establish properties of the dynamics update.
For simplicity of notation, we drop the time index and write the next state $p_+$ in terms of $p$ under input $q$ and step size $\eta$.
We first observe that since 
\[p_+ \propto p + \eta q^\top p q,\]
the updated preference $p_+$ always lies in the span of $p$ and $q$.
Therefore, $p_t$ lies in the span of $p_0$ and $q$ for all $t$.
We next argue that the sign of $p_t^\top q$ is invariant under the update.
\[\sgn(p_+^\top q) = \sgn(\tilde p_+^\top q) = \sgn(p^\top q + \eta p^\top q q^\top q) = \sgn((1+\eta)p^\top q) = \sgn(p^\top q)\:.\]
Since this holds for every update, the sign of $p_t^\top q$ is the same as $p_0^\top q$ for all $t$.
Therefore, not only is $p_+$ in the span of $p$ and $q$; it is specifically a positive combination of $p$ and $q$ (or $-q$).
As a result, $p_t$ lies in the cone defined by $p_0$ and $\sgn(p_0^\top q)\cdot q$ for all $t$.

Finally, we bound the convergence rate.
Note that $(p_+^\top q)^2 = \frac{(\tilde p_+^\top q)^2}{\|\tilde p_+\|_2^2}$.
Starting with the denominator,
\begin{align*}
    \|\tilde p_+\|_2^2 &= \|p\|_2^2 + 2\eta(p^\top q)^2 + \eta^2(p^\top q)^2\|q\|_2^2
    =(1+ \eta)^2 ( p^\top q)^2 - ( p^\top q)^2 + 1.
\end{align*}

A straightforward computation shows that
\begin{align*}
 (p_{+}^\top q)^{-2} -1 &=
    \frac{ 1}{ (\eta  + 1)^2} \left(( p^\top q)^{-2}- 1\right).
\end{align*}

This recursion
allows us to precisely characterize the trajectory of $p_t$.
In particular, we have that
\begin{align*}
 (p_{t}^\top q)^{-2} -1 &=
    \prod_{k=0}^{t-1} (\eta_k + 1)^{-2} \left(( p_0^\top q)^{-2}- 1\right).
\end{align*}
All that remains is simplifying the product for the two cases of step size.
For constant $\eta_t=\eta$, the expression simplifies to $(\eta  + 1)^{-2t}$,
while for decreasing $\eta_t = \frac{\eta}{t+\timeoffset}$
\begin{align*}
 \prod_{k=0}^{t-1} \left(\frac{\eta}{k+\timeoffset} + 1\right)^{-2}=
    \prod_{k=0}^{t-1}\left(\frac{k+\timeoffset} { \eta + k+\timeoffset}\right)^2
    =\prod_{k=0}^{\eta-1}\left(
    \frac{\timeoffset+k}{t+\timeoffset+k}
    \right)^2
\end{align*}
where the final equality holds because $\eta$ and $s$ are integers.
\end{proof}

We can even precisely describe the exact location of preferences under fixed recommendations, which we do below.
\begin{coro}\label{lem:alpha_beta}
For the fixed recommendation setting described in Proposition~\ref{prop:convergence},
\[p_t =\left( \frac{{\gamma_t} }{\sqrt{(q^\top p_0)^2+\gamma_t^2(1-(q^\top p_0)^2)}} \right) p_0 + \left( \frac{q^\top p_0(1-{\gamma_t}) }{\sqrt{(q^\top p_0)^2+\gamma_t^2(1-(q^\top p_0)^2)}} \right) q\:.\]
The magnitude of the first term is decreasing in $t$ and the second is increasing.
\end{coro}

\subsection{Performance of Fixed Recommendations}

We now show that the characterizations in the prior section lead to vanishing regret in the objective function of affinity of the  user to recommended items,
\[r_t^\affinity = q_t^\top p_t\:.\]
Though the ideal recommendation at any time $t$ would be aligned with $p_t$ exactly, the fast convergence result presented in the previous section suggests that achieving high performance is easy.
It is necessary only to require that the sign of the inner product of the individual's preferences and the item's representation is positive.
The following proposition formalizes this result by showing constant regret.

\begin{prop}\label{prop:fixed_regret}
If $q_t=q$ for all $t\geq 0$ and $p_0^\top q>0$, then the cumulative regret is  at most constant:
\[R(T)\leq  C_\gamma(( p_0^\top q)^{-2}- 1),\qquad C_\gamma= \begin{cases}   \frac{(\eta^2+1)}{\eta^2 + 2\eta  }& \eta_t=\eta\\ \frac{s^2\pi^2}{6}  & \eta_t= \frac{\eta}{t+\timeoffset}\end{cases}\]
\end{prop}

Notice that the regret is smallest when $p_0^\top q$ is close to $1$.
This does support common practice to select the $q\in Q$ which has maximum inner product with $p_0$, i.e. which has largest predicted affinity. 
However, the regret of any $q$ in the same hemisphere of $p_0$ will only differ by a constant factor.
Therefore, it does not take much to achieve small regret.

This has a direct implications for questions of polarization in this model under performant personalized recommendations.
In particular, under fixed personalized recommendations $q^u$ for each user $u$, the preference vector for user $u$ will converge to $p^u_t \to q^u$.
Whether or not this represents a polarized population depends on how the $q_u$ are chosen.
If they are chosen such that
\[q^u = \arg\max_{q\in Q} q^\top p_t^u\]
then the steady-state population level distribution of preferences is determined by the initial distribution as well as the richness of the available items $Q$.
For example, if there are fewer items than users, the variance of the converged distribution will be less than that of the initial distribution. 
Within local neighborhoods around each $q\in Q$, individuals' preferences will collapse towards $q$, losing their initial richness and variation.
We term this phenomenon \emph{mode collapse}.
Where polarization is a global population level phenomenon, mode collapse is a local loss of preference variation.

\begin{proof}[Proof of Proposition~\ref{prop:fixed_regret}]
We begin by deriving a relationship between the quantity which converges in Proposition~\ref{prop:convergence} and the reward.
\begin{align*}
   (p_{t}^\top q)^{-2} -1 =
    \frac{1-r_t^2}{r_t^2}
   \geq {1-r_t^2}
   \geq  1-|r_t|
\end{align*}
where we make use of the fact that $|r_t|<1$.
By assumption, $p_0^\top q>0$, so by Proposition~\ref{prop:convergence}, $r_t>0$ for all $t$.
Therefore, as long as $r_t>0$, 
\begin{align*}
   R(T) \leq \sum_{t=0}^{T-1}  (p_{t}^\top q)^{-2} -1  \leq  \left(( p_0^\top q)^{-2}- 1\right) \sum_{t=0}^{T-1} \gamma_t^2
\end{align*}
A straightforward computation shows that under fixed step-size, the the geometric sum is upper bounded by $\frac{(\eta+1)^2 }{\eta^2 +2\eta }$.
For decreasing step size,
\begin{align*}
\sum_{t=0}^{T-1} \gamma_t^2
&=\sum_{t=0}^{T-1} 
\prod_{k=0}^{\eta-1}\left(
    \frac{\timeoffset+k}{t+\timeoffset+k}
    \right)^2\leq
    \sum_{t=0}^{T-1} 
    \left(
    \frac{\timeoffset}{t+\timeoffset}
    \right)^2
    \leq s^2  \sum_{t=1}^{\infty}    \frac{1}{t^2}  = \frac{s^2\pi^2}{6}
\end{align*}
where the first inequality holds because the terms in the product are less than one.
\end{proof}

\subsection{Hemisphere Identification}

The previous section highlights the fact that a trivially simple approach to recommendation is sufficient for achieving high performance.
However, it assumes that at least the hemisphere of $p_0$ is known a priori. 
Does the problem become much harder when nothing is known about $p_0$?
In this section, we answer in the negative.
We present a simple explore-then-commit algorithm which identifies the hemisphere of $p_0$  with high probability, even when the partial observations are noisy.
Algorithm~\ref{alg:explor_commit} returns a fixed recommendation $q$ and achieves regret that scales logarithmically in the horizon $T$. 
The only requirement is that the set $Q$ contains two items pointing in sufficiently different directions relative to $p_0$.

\begin{algorithm}
        \caption{Explore-then-Commit for Hemisphere Identification}\label{alg:explor_commit}
        \begin{algorithmic}[1]
\Procedure{ExploreCommit}{$q_1,q_2,T_e,T$} 
            \State select $q_e=q_i$ where $i\in\{1,2\}$ each with probability $1/2$,  initialize $S=0$
            \For{$0\leq t\leq T_e$}
                \State recommend $q_t=q_e$, observe $y_t$,  update $S\leftarrow S+ y_t$.
            \EndFor
            \If{$S\geq 0$} select $q=q_i$
            \Else{} select $q=q_j$ for $j\neq i$
            \EndIf
            \For{$T_e\leq t\leq T$}
                \State recommend $q_t=q$
            \EndFor
        \EndProcedure
        \end{algorithmic}
    \end{algorithm}

\begin{prop}\label{prop:exp_commit}
Suppose that observations of affinity are noisy, with $y_t = p_t^\top q_t +w_t$ for $w_t$ i.i.d. subgaussian with parameter $\sigma^2$. 
Further suppose that $q_1,q_2\in Q$ are such that $q_1^\top q_2 < 0$, $q_1^\top p_0 >0$, and $q_2^\top p_0<0$ or vice-versa and let $a=\min\{|p_0^\top q_1|, |p_0^\top q_2|, |q_1^\top q_2|\}$.

Then for $T_e=\sigma^2 \log(T)/a^{2}$,
Algorithm~\ref{alg:explor_commit} incurs expected regret bounded by
\[\E[R(T)] \leq 2 + C_\gamma/a^4 + \sigma^2 \log(T)/a^2 \]
where $C_\gamma$ is a step-size dependent constant defined in Proposition~\ref{prop:fixed_regret}.
\end{prop}

We make a few comments on the assumptions and possible extensions.
Notice that the requirements $q_1^\top q_2 < 0$, $q_1^\top p_0 >0$, and $q_2^\top p_0<0$ (or vice-versa) are satisfied for every $p_0\in\mathcal S^{d-1}$ if and only if $q_1=-q_2$.
However, for any particular $p_0$, it need not be that an item and its negative are both contained within $Q$.
The dependence on the parameter $a$ which lower bounds the magnitude of these inner products can be understood as a ``gap'' dependence.
However, this gap is about the informativeness of our actions, rather than their reward.
It would be interesting to extend this algorithm to the case that such a $q_1$ and $q_2$ are not known a priori, and must be selected from $Q$ online.
We additionally note that an extension to arbitrary time horizons could be achieved by implementing a switching rule based on a running sum $S_t=\sum_{k=0}^t y_k$.
However, such extensions would need to reason carefully about switched nonlinear dynamics.

\begin{proof}[Proof Sketch]
The full proof appears in Appendix~\ref{app:main_res}.
Denote $q_\star$ as the recommendation that has positive inner produce with $p_0$.
Then bounding the expected regret involves several terms:
\[\E[R(T)] \leq \E[R_{q_\star}(T)] + \E[R_e(T)\mid q_e=q_\star] + \E[R_e(T)\mid q_e\neq q_\star] \]
The first is the regret of playing $q_\star$ at all time. Proposition~\ref{prop:fixed_regret} bounds this quantity.
Bounding the second and third requires reasoning about the probability that the exploration phase fails, so that $q\neq q_\star$.
The third term additionally requires us to reason about the different preference trajectories induced by playing $q_1$ vs. $q_2$ during the exploration phase.
Unlike arguments made for explore-then-commit algorithms in memoryless settings, we must account for the statefulness of preferences.
\end{proof}

\section{Stationary Preferences as A Design Goal}\label{sec:rand}

In the previous section, we considered the goal of recommending high-affinity items, and showed under our preference dynamics model this is a trivially easy task. We further show that though it may not lead to polarization, it can lead to a form of ``mode collapse'' where individuals' preferences converge to the closest item in  $Q$.
All this points to affinity maximization as an unsatisfying goal.

We therefore motivate a new definition of performance which focuses on \emph{non-manipulation} by rewarding \emph{stationary preferences}:
\[r_t^\stationary = p_t^\top p_0\:.\]
Notice that as in the previous setting, $q_t^\star = p_0$ would ensure that $r_t^\stationary=1$ for all $t$. 
However, it may be that of a finite set of items available for recommendation, none is exactly equal to $p_0$.
A natural choice would then be to select the closest item $q=\arg\max_{q\in Q}q^\top p_0$.
However, this fails to achieve the non-manipulation goal, as is made formal in the following proposition.

\begin{prop}
Under fixed recommendations $q_t=q\neq p_0$,
$R(T) \geq CT$ for a constant $C>0$ depending on $p_0^\top q$ in both the constant and decreasing step size settings.
\end{prop}

\begin{proof}
Using the expression from Corollary~\ref{lem:alpha_beta},
\[p_0^\top p_t =  \frac{{\gamma_t} + (q^\top p_0)^2(1-{\gamma_t})}{\sqrt{(q^\top p_0)^2+\gamma_t^2(1-(q^\top p_0)^2)}}  \:.\]
This is a decreasing function in $\gamma_t$ by Lemma~\ref{lem:inc_dec}; hence for $t\geq 1$ it can be upper bounded by $\gamma_1$.
Letting $0<C_1<1$ be the right hand side of the expression at $\gamma_t=\gamma_1$, we lower bound the regret by
\[R(T) =\sum_{t=0}^{T-1} 1 - p_0^\top p_t \geq \sum_{t=1}^{T-1} 1-C_1 \geq \frac{1-C_1}{2}T \:.\]
\end{proof}

While in principle a time-varying recommendation policy could be computed on any finite horizon using dynamic programming, in practice the expressions quickly become unwieldy due to maximizing a nonlinear expression over the discrete set $Q$.
Furthermore, it is not clear how to relax the dependence of such a policy on the time horizon.
In the next sections, we show that a randomized policy can do much better in (approximately) leaving preferences unchanged.

\begin{figure}
    \centering
    \includegraphics[width=0.49\textwidth]{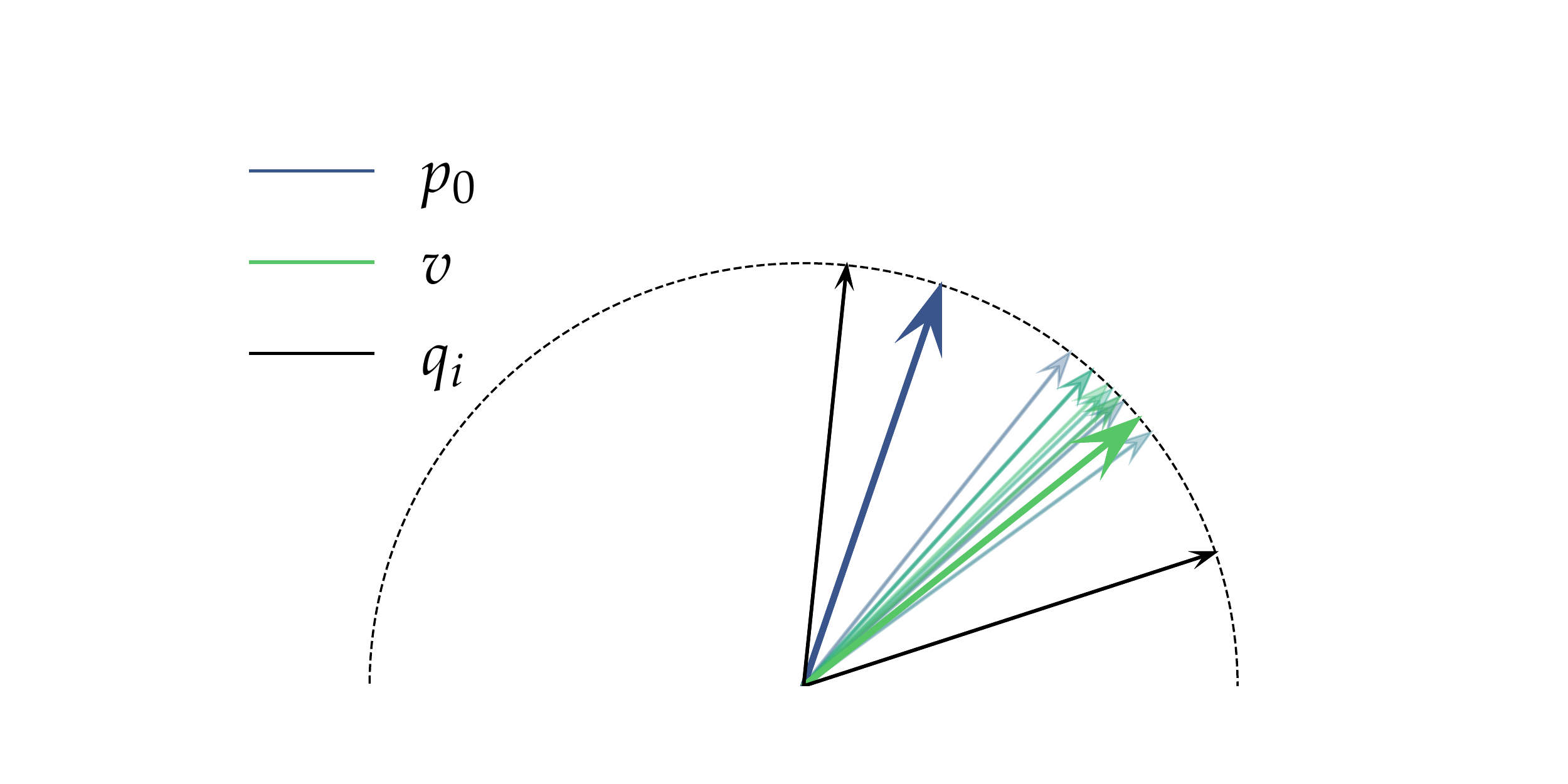}~\includegraphics[width=0.49\textwidth]{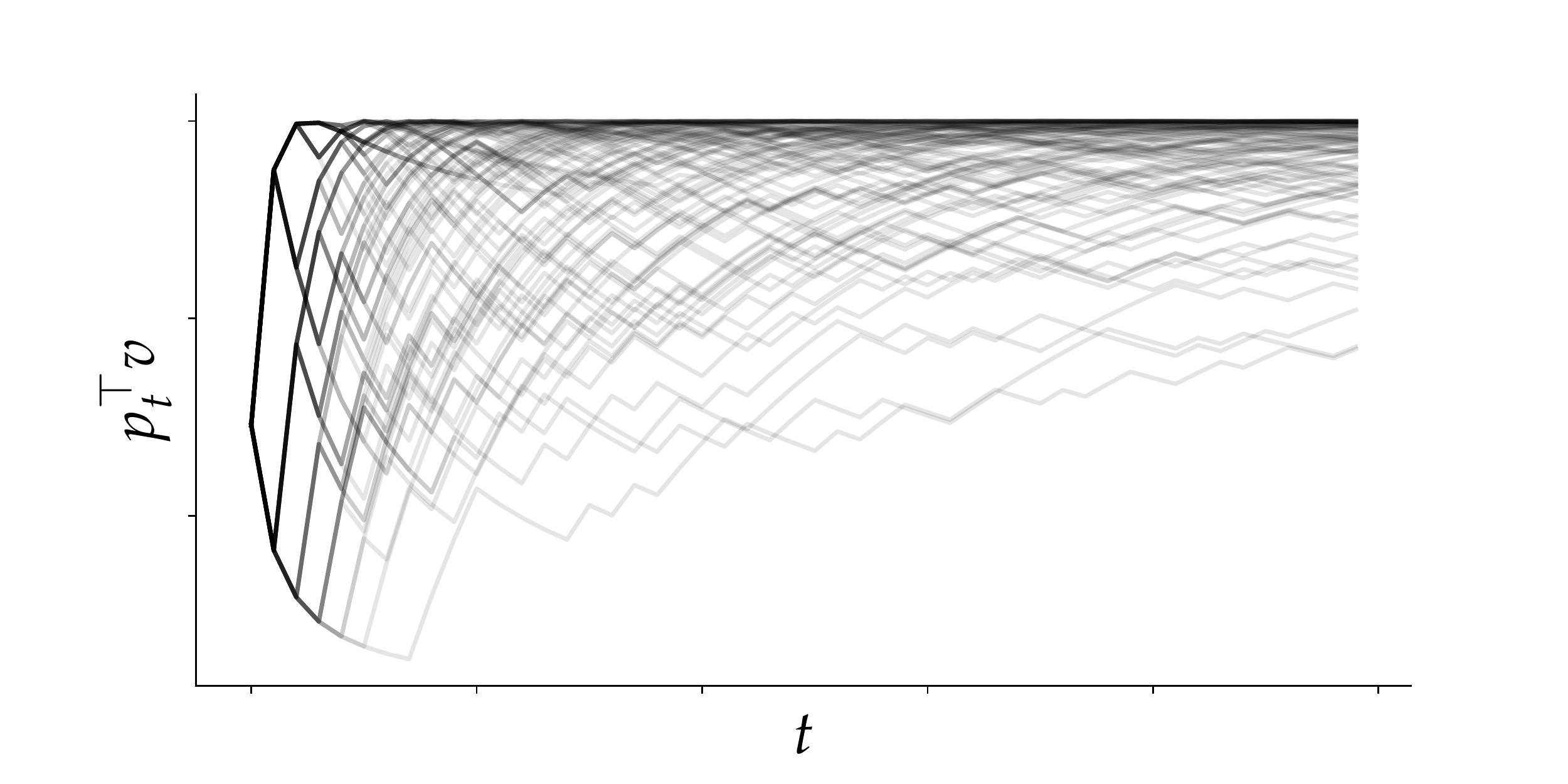}
    \caption{Left: an illustration of preference evolution from $p_0$ (blue) towards $v$ (green) under randomized recommendation over $q_1$ and $q_2$. Right: a plot of the inner product betwen $p_t$ and $v$ over time for 100 trials.}

    \label{fig:rand}
\end{figure}

\subsection{Randomized Recommendations}

Motivated by the stationary preferences objective, in this section we develop a method for guaranteeing that preference $p_t$ converges to arbitrary $v_1\in\mathcal S^{d-1}$.
We make use of randomization over recommendations of items from $Q$.
Our method is inspired by the algorithmic approaches used in streaming PCA, specifically Oja's method, whose iterates evolve analogously to the preference model we study~\cite{de2015global,jain2016streaming}.

Our algorithm maintains a probability distribution over the discrete set $Q$ and draws items independently from that distribution. 
In this section, we present an analysis which characterizes the preference trajectories under such a scheme.
We focus on the decreasing step size setting, as the following remark argues that constant step size  won't converge under randomization.

\begin{remark}
When the step size $\eta_t$ is constant, the preference $p_t$ will not converge under randomized recommendations. 
This is straightforward to see by considering the effect of any particular $q_t$ which will cause $p_{t+1}$ to move $\eta q_t^\top p_t$ towards it.
However, for small enough step size, it is possible to argue that $p_k$ will be close to $v_1$ for some $k\leq t$ with probability scaling like $1-\mathcal O(1/t)$;
similar arguments appear in the context of streaming algorithms~\cite{de2015global}.
\end{remark}

We consider a distribution over $N$ items defined by the probability weighting $\alpha\in\R_+^N$ with $\sum_{i=1}^N \alpha_i=1$. At each timestep $t$, $q_t$ is selected to be $q_i\in Q$ with probability $\alpha_i$.
Under this stochastic rule, the expected dynamics are defined by the covariance matrix
\[\E[\tilde p_{t+1}] = (I + \eta_t \E[q q^\top])p_t\:. \]
Let us refer to this matrix as $\Sigma = \E[q q^\top] = \sum_{i=1}^N \alpha_i q_i q_i^\top$. 
Let $v_i$ be the eigenvector of $\Sigma$ associated with $\lambda_i$ for $\lambda_1\geq \lambda_2\geq \dots \geq \lambda_N$.
Then the eigen-decomposition allows us to write
\begin{align}\label{eq:definev}
    \sum_{i=1}^N \alpha_iq_iq_i^\top=\Sigma = V\Lambda V^\top,\quad V = \begin{bmatrix} v_1 & V_{2:}\end{bmatrix}
\end{align}
where $\Lambda=\diag(\lambda_1,\dots,\lambda_N)$, $V$ has $v_i$ in column $i$, and we use $V_{2:}$ to denote the matrix containing the non-dominant eigenvectors.
Define the following measure of spread
\begin{align}\label{eq:spread}
    M=\max\left\{\max_{j\in[N]}\|q_jq_j^\top - \Sigma\|, \lambda_1\right\}\:.
\end{align}

We present a result which shows that under this randomization scheme, the preference $p_t$ converges to the dominant eigenvector of the covariance matrix.
It follows by a straightforward application of Theorem 12 from~\citet{huang2021streaming}. We provide a proof in Appendix~\ref{app:main_rand_res}.
\begin{coro}\label{prop:random_conv}(Theorem 12 from~\citet{huang2021streaming})
Consider preference dynamics with step size $\eta_t=\frac{\eta}{t+\timeoffset}$ and randomized recommendations with $q_t$ chosen from $Q$ at each $t$ according to probabilities $\alpha$.
Define $v_1$ as in~\eqref{eq:definev} and assume that $ v_1^\top p_0\geq \sqrt{2}/2$.
Further assume that
\[\eta\geq \frac{8}{\lambda_1-\lambda_2}\quad \text{and}\quad s\geq 1+2C^2\eta^2M^2\log\left(2CMT\eta/\delta\right),\]
where $C$ is a fixed constant less than $175$, $M$ is defined in~\eqref{eq:spread}, and $T\geq 0$.
Then with probability  $1-\delta$, for all $t\leq T$,
$$1-(p_t^\top v_1)^2\leq 4e^2 \frac{s+2}{s+1+t}\:.$$
\end{coro}

This result shows high probability convergence of $p_t$ to $v_1$, the dominant eigenvector of $\Sigma$.
It is the consequence of a straightforward translation of
matrix concentration results developed by~\citet{huang2021streaming} in the context of a streaming PCA algorithm.
A few remarks on the assumptions: to some extent, the lower bounds required on $v_1^\top p_0$, $\eta$, and $s$ could be relaxed at the expense of simplicity in the final expression.
Smaller $\eta$ leads to a slower convergence rate, with a term proportional to $\lambda_1-\lambda_2$ appearing in the exponent.
Furthermore, lower bounds on $v_1^\top p_0$ and $s$ could be effectively achieved with an initialization routine and re-indexing of time. 
The routine would simply apply some fixed recommendation $q_v$ with $v_1^\top q_v> \sqrt{2}/2$ for a finite number of steps before switching to randomized recommendations.

Figure~\ref{fig:rand} shows an example illustrating the high probability convergence.
Trivially, if $\alpha$ puts all mass on a particular $q_i$, then $v_1=q_i$ and convergence also follows from our fixed recommendation analysis in Proposition~\ref{prop:convergence}.
However, since we are free to choose from many possible probability weightings, this strategy allows for a much rich set for $p_t$ to converge to.
Randomization thus provides a potential solution to the problem of mode collapse.
The following section explores conditions on $Q$ that guarantee when a particular $v$ can be set as the dominant eigenvector for an appropriate choice of $\alpha$.

We conclude this section by connecting the convergence result to the performance in terms of the stationary preference objective defined by $r_t^\stationary $.
This result requires that the sign of the inner product $p_t^\top v_1$ remains positive, which motivates the following definition.

\begin{definition}
We say that a set of vectors $u_1,\dots, u_k$ is \emph{self-aligned} with respect to $v$ if $u_i^\top u_j\geq 0$ for all pairs $i,j$ and $u_i^\top v\geq 0$ for all $i$. 
\end{definition}

\begin{coro}\label{coro:rand_regret}
Consider the setting and assumptions of Corollary~\ref{prop:random_conv} and
further assume that support of the randomization is restricted to a set of items which is self-aligned with respect to $v_1$ and $p_0$.
Then, the regret of this recommendation approach is bounded
with probability $1-\delta$,  
\[R(T) \leq 4e^2 (s+2)\left(1+\log\left(\tfrac{T+s}{s}\right)\right) + T\|v_1-p_0\|    \]
\end{coro}
Notice that when $v_1=p_0$ the bound is logarithmic in $T$.
Furthermore, an initialization routine with fixed recommendations (to ensure the lower bound on $s$) would contribute an amount to the regret scaling no more than logarithmically in $T$.

\begin{proof}
First notice that at any time $k$, $q_k^\top p_0>0$ with probability $1$ by the self-aligned assumption. 
Assume that the same holds for $p_{t-1}$. 
Then we show that for $p_t$,
\[q_k^\top \tilde p_{t} = q_k^\top  p_{t-1}+\eta_{t-1} q_k^\top q_{t-1} \cdot q_{t-1}^\top  p_{t-1}\]
Due to the self-aligned assumption, $q_k^\top q_{t-1}\geq 0$, and by inductive assumption the other terms are also positive.
Therefore, we have shown by induction that $q_k^\top p_t>0$ for all $k$ and $t$.

Next, notice that $v_1^\top p_0>0$ by assumption.
Assume the same holds for $p_{t-1}$.
Then for $p_t$,
\[v_1^\top \tilde p_{t} = v_1^\top  p_{t-1}+\eta_{t-1} v_1^\top q_{t-1} \cdot q_{t-1}^\top  p_{t-1}\]
Due to the self-aligned assumption, $v_1^\top q_{t-1}>0$, and the other terms are also positive by inductive assumption and the previous paragraph. 
Therefore, $p_t^\top v_1\geq 0$ and thus
$p_t^\top v_1 \geq (p_t^\top v_1 )^2$.

Therefore, we can bound the regret as
\begin{align*}
    R(T) %
    &\leq \sum_{t=0}^{T-1}1-(v_1^\top p_t)^2 + \|v_1-p_0\|
    \leq  4e^2 \sum_{t=0}^{T-1}\frac{s+2}{s+1+t} + T\|v_1-p_0\|
\end{align*}
where the final inequality uses the convergence result in Corollary~\ref{prop:random_conv}.
The result follows by bounding the harmonic series
\begin{align*}
    \sum_{t=0}^{T-1}\frac{1}{s+1+t} 
    = \sum_{t=1}^{T+s}\frac{1}{t} - \sum_{t=1}^{s}\frac{1}{t} \leq 1 + \log(T+s) - \log(s)\:.
\end{align*}
\end{proof}

\subsection{Designing Randomization}

The previous section develops a randomized approach to recommendation that allows preferences $p_t$ to converge to a richer set of directions $v\in\mathcal S^{d-1}$ than the individual items available for recommendation from the fixed set $Q$.
In this section, we precisely characterize how rich the set of possible $v\in\mathcal S^{d-1}$ is in terms of the item set $Q$.
In particular, we develop a method based on convex optimization for designing randomization schemes.

With a slight overloading of notation, we will use $Q$ to denote the ${\latentdim\times\numitems}$ matrix whose columns are given by $q_i$.
The following result presents convex sets which characterize the probability weightings $\alpha$ that ensure convergence to a given $v\in\mathcal S^{d-1}$.
Its proof is presented in Appendix~\ref{app:main_rand_res}.
\begin{prop}\label{prop:eigen}
For a given $v\in\mathcal S^{d-1}$,
define the sets
\begin{align*}
    \mathcal C_\mathrm{eig} &= \left\{x\in \R^\numitems\mid x\geq 0,\, Q\diag(Q^\top v) x = v\right\}\\
    \mathcal C_\mathrm{dom} &= \left\{x\in \mathcal C_\mathrm{eig}(v)\mid I - Q\mathrm{diag}(x)Q^\top \succeq 0\right\}
\end{align*}
For any $x\in\mathcal C_\mathrm{eig}$,
setting the weights as $\alpha_i = \frac{x_i}{ 1^\top x}$ ensures that $v$ is an eigenvector of $\E[qq^\top]=\Sigma$ with eigenvalue $(1^\top x)^{-1}$. If $x$ is also in $\mathcal C_\mathrm{dom}$, then it is additionally the dominant eigenvector.
Furthermore, $C_\mathrm{eig}$ characterizes all possible weightings $\alpha_i$ that ensure that $v$ is an eigenvector with nonzero eigenvalue.
\end{prop}

These sets are convex, and thus the feasibility of finding such an $x\in\mathcal C_{\mathrm{dom}}$ can be checked efficiently using convex solvers. 
However, equivalent conditions provide further geometric intuition about what properties are necessary of the set $Q$ of items available to recommend.
To state this condition, we define
\begin{align}\label{eq:Qv}
    \bar q_i =\sgn(v^\top q_i)\cdot q_i,\quad Q_v = \{\bar q_i\}_{i=1}^\numitems\:.
\end{align}
This set is constructed by
replacing the vectors in $Q$ with their negatives if they have negative inner product with $v$. It is therefore contained in the dual cone defined by $v$. 

\begin{prop}
Define ${Q}_v$ as in~\eqref{eq:Qv}. For any fixed $v\in\mathcal S^{d-1}$, 
\begin{enumerate}
    \item the set $\mathcal C_\mathrm{eig}$ is nonempty  if and only if  $v$ is in the conical hull of $Q_v$. That is, there exist weights $w_i\geq 0$ such that
    $v = \sum_{i=1}^N w_i \tilde q_i$.
    \item the set $\mathcal C_\mathrm{dom}$ is non-empty if and only if the weights $w_i$ satisfy, for nonzero weight indices $\mathcal I_+ = \{i\mid w_i>0\}$,
        $$\diag(\{|q_i^\top v|\}_{i\in\mathcal I_+}) \succeq \diag(\{w_i\}_{i\in\mathcal I_+}) Q_{\mathcal I_+}^\top Q_{\mathcal I_+}.$$
\end{enumerate}
\end{prop}
The matrix on the right hand side of the final condition can be interpreted as a column-weighted covariance of item vectors. 
A sufficient condition for the second condition is that for all $i\in\mathcal I_+$,
    $$w_i \cdot \sum_{j\in\mathcal I_+} |q_i^\top q_j|\leq |q_i^\top v|.$$
    This follows from the Gershgorin circle theorem.
    The dominant eigenvector condition can therefore be understood as requiring that the positive weight on each $\bar q_i\in Q_v$ is not too large, relative to the magnitude of the angle between $q_i$ and $v$: in other words, we need to be able to achieve the conical weighting based largely on aligned items.
We remark that for the regret bound in Corollary~\ref{coro:rand_regret} to hold, it is furthermore necessary for the set of items with positive weight to be self-aligned with respect to $v$.
The problem of selecting a self-aligned subset of $Q$ can be expressed as a bilinear condition on $x$: $\mathcal C_\mathrm{conv} = \left\{x\in \mathcal C_\mathrm{dom}\mid x_i x_j q_i^\top q_j \geq 0, x_i q_i^\top v \geq 0\right\}$.
Bilinearity precludes directly incorporating this condition into a convex solver. 
However, subset selection heuristics (e.g. all $q_i$ such that $q_i^\top v$ is large) are likely to be successful for sufficiently rich sets of items $Q$.

\begin{proof}
Notice that the conditions of $\mathcal C_\mathrm{eig}$ are equivalent to $v$ being in the conical hull of the columns of $Q\diag(Q^\top v)$.
These columns are given by $\{q_i^\top v \cdot q_i\}_{i=1}^N$.
Therefore, $\mathcal C_\mathrm{eig}$ is nonempty if and only if
\[v=\sum_{i=1}^N x_i q_i^\top v \cdot q_i=\sum_{i=1}^N \underbrace{x_i |q_i^\top v|}_{w_i} \sgn(q_i^\top v) q_i\]
for some $x\geq 0$. This is exactly the conical hull of $Q_v$.

Turning to the final condition in $\mathcal C_\mathrm{dom}$, we have by a Schur complement argument that
\begin{align*}
    I - Q\mathrm{diag}(x)Q^\top \succeq 0 &\iff I - Q_{\mathcal I_+}\mathrm{diag}(x)_{\mathcal I_+}Q_{\mathcal I_+}^\top \succeq 0
    \\ &\iff \begin{bmatrix} I & Q_{\mathcal I_+}\\ 
Q_{\mathcal I_+}^\top& \mathrm{diag}(x)_{\mathcal I_+}^{-1} \end{bmatrix} \succeq 0 \iff \mathrm{diag}(x)_{\mathcal I_+}^{-1} - Q^\top Q \succeq 0
\end{align*} 
Replacing $x_i$ with $w_i/|q_i^\top v|$ as above, the result follows.
\end{proof}

Now that we have developed an understanding of when it is feasible to set the dominant eigenvector to some $v$,
we turn to designing the probability weights such that the eigengap $\lambda_1-\lambda_2$ is maximized.
First, recall that the sum of the top $k$ eigenvalues of a matrix $A$ can be written as a convex function, in the form of a semi-definite program
\[f_k(A) = \sum_{i=1}^k \lambda_i(A) =  \min_{s,t,Z} ~~ t\quad \text{s.t.}\quad t-ks - \mathrm{tr}(Z)\geq 0,~~ Z\succeq 0,~~ Z - A + sI \succeq 0\:.\]
Then notice that the eigengap of a matrix can be written as
$\lambda_1(A)-\lambda_2(A) = 2\lambda_1(A)-f_2(A)$.
Therefore, we can formulate an optimization problem for finding probability weights which set $v$ as the dominant eigenvector and maximize the eigengap.
Consider the following concave semi-definite program:
\begin{align}
\begin{split}\label{eq:SDP}
    \max_{x}~~& \frac{1}{1^\top x}(2-f_2(Q\mathrm{diag}(x)Q^\top))\\
    \text{s.t.}~~&x \geq 0,~~ Q \mathrm{diag}(Q^\top v) x = v,~~ I - Q\mathrm{diag}(x)Q^\top \succeq 0
\end{split}
\end{align}
It is concave because the objective is linear fractional.
Under the idenfication that $\lambda_1=(1^\top x)^{-1}$ and $\alpha = x/(1^\top x)$, the objective maximizes the eigen-gap of $A=Q\diag(\alpha)Q^\top$.

\subsection{Initial Preference Identification}
The randomization methodology developed in the previous two sections can be used to design recommendations which keep preferences stationary by selecting $v=p_0$.
However, this assumes that $p_0$ is directly observed, which is often not the case.
We therefore outline conditions under which $p_0$ can be identified from observations of the form $y_t = p_t^\top q_t$.
We will make use of the following notation to write the dynamics: define transfer matrices $\Phi_0=I$ and $\Phi_{t+1} = (I+\eta_t q_t q_t^\top)\Phi_t$ so that $\tilde p_t=\Phi_t p_0$.

Given a sequence of recommendations $q_0,\dots,q_{T-1}$, define the function $F_T:\mathcal S^{d-1}\to \R^T$
which generates the sequence of observations
\[\begin{bmatrix}y_0\\\vdots \\ y_{T-1}\end{bmatrix} = \begin{bmatrix}q_0^\top p_0\\\vdots \\ q_{T-1}^\top p_{T-1} \end{bmatrix}= \begin{bmatrix}\frac{q_0^\top \Phi_0 p_0}{\|\Phi_0p_0\|}\\\vdots \\ \frac{ q_{T-1}^\top\Phi_{T-1} p_0}{\|\Phi_{T-1} p_0\|} \end{bmatrix}=: F_T(p_0) \]
Then the problem of identifying initial preferences is related to this invertibility of this nonlinear observation function.

\begin{prop}
If $q_0,\dots, q_{T-1}$ span $\R^d$, the function $F_T$ is locally invertible for any $p_0\in\mathcal S^{d-1}$.
\end{prop}
\begin{proof}
A differentiable function on a manifold is locally invertible around a point $p_0$ if
the \emph{differential} evaluated at $p_0$ is invertible~\cite[Theorem 4.16]{boumal2020introduction}.
The differential of a function  at $p_0$ is invertible
if the nullspace of the (Euclidean) Jacobian is perpendicular to the tangent of the manifold in which $p_0$ resides.
Translating for the case of the unit sphere, we check whether the nullspace of the Jacobian is parallel to $p_0$.
The Jacobian of $F_T$ is given by
\[J = \begin{bmatrix}\nabla_{p_0} y_0^\top \\ \vdots \\ \nabla_{p_0} y_{T-1}^\top\end{bmatrix},\quad\text{where}\quad\nabla y_t = \frac{1}{\|\Phi_t p_0\|} \left(\frac{\Phi_t^\top \Phi_t p_0p_0^\top}{\|\Phi_t p_0\|^2}-I \right)\Phi_t^\top q_t\]

Consider an arbitrary vector $v\in\R^d$ such that $Jv = 0$. In other words $\nabla_{p_0} y_t ^\top v = 0$ for all $t$.
Simplifying that expression,
\begin{align*}
    \nabla_{p_0} y_t ^\top v & = \frac{1}{\|\Phi_t p_0\|} q_t^\top \Phi_t \left(\frac{p_0p_0^\top\Phi_t^\top \Phi_tv }{\|\Phi_t p_0\|^2}-v \right)
    = \frac{1}{\|\Phi_t p_0\|} q_t^\top  \left(p_t p_t^\top v_t - v_t\right)
\end{align*}
where we define $v_t = \Phi_t v$ and recall that $p_t = \frac{\Phi_t p_0}{\|\Phi_t p_0\|}$.
Therefore, we have derived that
\[\nabla_{p_0} y_t v = 0 \iff q_t^\top  \left(p_t p_t^\top v_t - v_t\right) = 0 \]
Define  $v^\perp = (I- p_0 p_0^\top)v$ and consider the following claim:
\[\nabla_{p_0} y_t^\top v=0~~\forall~t\iff q_t^\top v_\perp = 0~~\forall~t.\]
If this claim is true, then 
$v_\perp=0$ since $q_1,\dots,q_{T-1}$ span $\mathbb R^d$ by assumption. As a result, $v \propto p_0$, meaning that the nullspace of $J$ is parallel to $p_0$ and thus $F_T$ is locally invertible.
None of the argument relied on anything in particular about $p_0$, so local invertibility is implied for any $p_0$.

It remains only to prove the claim, which we do by induction.
The base case of $t=0$:
\[q_0^\top  \left(p_0 p_0^\top v - v\right) = 0 \iff q_0^\top v^\perp = 0 \]
follows because $v = p_0 p_0^\top v + v^\perp$.

Define $v^\perp_t = (I- p_t p_t^\top)v_t$ which is perpendicular to $q_t$. 
Then a straightforward calculation confirms  that $q_t^\top v^\perp_t = 0$ implies $v^\perp_{t+1}=v^\perp_t$, using the fact that $p_t^\top v^\perp_t=0$.
Then by induction $v^\perp_t=v^\perp$ for all $t$, which proves the claim.
\end{proof}

Ensuring that the observation function is locally invertible allows for the possibility of estimating $p_0$ from affinity observations.
A natural approach for estimation would be an iterative method, for example Newton's method for manifolds~\citep{dedieu2003newton}.
Analyzing the convergence of such an approach, especially in the presence of noisy observations, would allow for performance guarantees for an explore-then-commit style algorithm based on the randomized method presented in previous sections.
We leave the design and analysis of such an estimation algorithm to future work.

\section{Conclusion and Discussion}
In this work, we investigate the concept of preference drift 
in personalization systems, where a user's preferences become more aligned with content they see and like, and more anti-aligned with content they see and dislike. We analyze the dynamics of such preferences when a content recommendation policy aims to maximize reward. We show that any algorithm which eventually shows users content within their same hemisphere will enjoy vanishing average regret. These dynamics reward policies which learn very little about user preferences, and instead change those preferences to be more aligned with the content the platform shows them. 
In scenarios where manipulation of user preferences  is not preferable or acceptable, policies which do not change user preferences by very much could have considerable appeal. We show how to construct such ``stationary'' policies when the set of content available is sufficiently rich. 
We further investigate when it is possible to learn enough about a user's preferences in a partial observation setting to employ this approach.

We now reflect on connections and distinctions between this work and that on polarization of preferences in the absence of personalization recently explored by ~\citet{hkazla2019geometric} and~\citet{gaitonde2021polarization}.
On the technical side, our work develops finite-time statements of convergence and performance, while the aforementioned works  show polarization as $t\to\infty$. Many approaches, including ours, would imply finite horizon (approximate) polarization under non-personalized recommendations. 
For a fixed recommended item $q$, our result in Proposition~\ref{prop:fixed_regret} implies that all individuals will polarize around either $q$ or $-q$.
For randomized recommendations such that $\E[qq^\top]=\Sigma$ has a dominant eigenvector $v_1$ and decreasing step size, our result in Proposition~\ref{prop:eigen} immediately implies finite-time convergence to $\pm v_1$.
Deriving finite-time statements of polarization for the uniformly random recommendations studied in the aforementioned works would require further analysis, for example by a stopping time argument.
We further comment that in the decreasing step size setting $\eta_t\propto \frac{1}{t+1}$, it is not clear whether the strong form of polarization studied by ~\citet{hkazla2019geometric} and~\citet{gaitonde2021polarization} holds.

More broadly, our conclusions on the effects of biased assimilation is distinct from previous work.
Personalized recommendations decouple the polarization dynamic induced by biased assimilation.
Rather than the convergence of populations towards poles, our analysis demonstrates the convergence of individuals towards the content they consume.
This brings to bear questions about the set of available content and the objectives of recommendation systems.
Our results demonstrate that the class of approximately optimal recommendation policies is vastly different when considering affinity maximization vs. stationary preferences.
Interestingly, the randomization technique we use in the context of stationary preferences resonates with techniques for ensuring exploration or diversity among recommended items. 

We conclude with additional ideas for future work. 
There are many interesting technical questions which arise from this dynamics model.
One immediate extension to this work is to develop an algorithm for achieving stationary preferences in the noisy partial observation setting, which can be viewed as a nonlinear state estimation problem.
Another interesting direction would be to consider questions of representation. What if both item vectors $q$ and user vectors $p$ must be learned from data or approximated in lower dimensions? This would align with practice in the recommendation systems community.
Finally, there are many interesting ways to extend the dynamics model itself. 
For example, rather than purely a response to external stimuli, preferences may evolve as a result of internal processes even in the absence of recommendations.
How would this change optimal recommendation strategies?
And if they are not known a priori, how might such dynamics be estimated from data?
Another extension to the model could decouple the observable behaviors from true preferences and affinities.
This would effectively relax a ``revealed preference'' assumption, taking a cue from work on inconsistent preferences~\cite{strotz1955myopia,kleinberg2022challenge}.

\subsection*{Acknowledgements}
This work was supported by an NSF CAREER award (ID 2045402),  An NSF AI Center Award (The Institute for Foundations of Machine Learning), and the Simons Collaboration on the Theory of Algorithmic Fairness.

\bibliographystyle{plainnat}
\bibliography{main}
\newpage
\appendix

\section{Proofs of results in Section~\ref{sec:conv}}\label{app:main_res}

\subsection{Preference Trajectory for Corollary~\ref{lem:alpha_beta}}
\begin{proof}[Proof of Corollary~\ref{lem:alpha_beta}]
Without loss of generality, assume that $q^\top p_0>0$. Otherwise, the same analysis holds replacing $q$ with $-q$.
Then by Proposition~\ref{prop:convergence}, $q^\top p_t>0$ for all $t$ and $p_t$  is a positive weighted combination of $p_0$ and $q$.
We can therefore write
\[p_t = \alpha p_0 + \beta q\]
We have that
\[q^\top p_t = \alpha q^\top p_0 + \beta  ,\quad \|p_t\|_2^2 = \alpha^2 + \beta^2 + 2\alpha\beta p_0^\top q = 1\:.\]
Combining the two equations yields
\[1=\alpha^2 + (q^\top p_t - \alpha q^\top p_0)^2 + 2\alpha (q^\top p_t - \alpha q^\top p_0) p_0^\top q = (1-(q^\top p_0)^2)\alpha^2 + (q^\top p_t)^2\]
Therefore,
\[\alpha^2 = \frac{1-(q^\top p_t)^2}{1-(q^\top p_0)^2} = \frac{1-\frac{1}{1+\gamma_t^2\frac{1-(q^\top p_0)^2}{(q^\top p_0)^2}}}{1-(q^\top p_0)^2} = \frac{\gamma_t^2}{(q^\top p_0)^2+\gamma_t^2(1-(q^\top p_0)^2)} \]
Plugging in this expression,
\begin{align*}\beta &= q^\top p_t - \frac{\gamma_t}{\sqrt{(q^\top p_0)^2+\gamma_t^2(1-(q^\top p_0)^2)}} q^\top p_0\\
&= 
 \frac{q^\top p_0}{\sqrt{(q^\top p_0)^2+\gamma_t^2 (1-(q^\top p_0)^2 ) }} - \frac{\gamma_t q^\top p_0}{\sqrt{(q^\top p_0)^2+\gamma_t^2(1-(q^\top p_0)^2)}} 
 \end{align*}
 Finally, we apply Lemma~\ref{lem:inc_dec} with $x=\gamma_t$ and $a=q^\top p_0$ to notice that $\alpha$ increases with $\gamma_t$ while $\beta$ decreases. 
Since $\gamma_t$ is decreasing in $t$, this means that $\alpha$ decreases with $t$ and $\beta$ increases.
\end{proof}

\subsection{Regret Argument for Proposition~\ref{prop:exp_commit}}
\begin{proof}[Proof of Proposition~\ref{prop:exp_commit}]
We make the argument in the case that $q_1^\top p_0>0$; reversing the roles of $q_1$ and $q_2$ will yield the same final bound.
Then we break the regret into two terms
\[R(T) \leq R_{q_1}(T) + R_e(T) \]
where $R_{q_1}$ is the regret incurred by playing $q_1$ for all $t$.
This can be bounded following Proposition~\ref{prop:fixed_regret}. $R_e(T)$ is the regret of the exploration algorithm compared with playing $q_1$ at all time.

Define $\bar p_t$ as the sequence of preference vectors under the policy $q_t=q_1$ for all $t$ and notice that $\bar p_t^\top q_1>0$ for all $t$ by sign invariance from Proposition~\ref{prop:convergence}.
Then the regret of Algorithm~\ref{alg:explor_commit} compared to playing $q_t=q_1$ is
\begin{align*}
    R_e(T) = \sum_{t=0}^{T-1} q_1^\top \bar p_t - q_t^\top p_t&= %
    \sum_{t=0}^{T_e-1} q_1^\top \bar p_t - q_e^\top p_t +  \sum_{t=T_e}^{T-1} q_1^\top \bar p_t - q^\top p_t\:.
\end{align*}

First consider the case that $q_e=q_1$. Then the first summation is equal to zero since $\bar p_t=p_t$ while $q_t=q_1$.
The second summation is also zero if we choose $q=q_1$. 
However, it is possible that at the end of the exploration phase we mistakenly switch to $q_2$ based on observations. 
It will be $q=q_2$ if $S<0$. Thus, 
\begin{align*}
    \Pr[q=q_2]%
    &= \Pr\left[\sum_{t=0}^{T_e-1} q_1^\top p_t < -\sum_{t=0}^{T_e-1} w_t\right]\leq e^{-(\sum_{t=0}^{T_e-1} q_1^\top p_t)^2/T_e\sigma^2} \leq e^{-(p_0^\top q_1)^2T_e/\sigma^2}
\end{align*}
where we use that the noise is sub-Gaussian and that $q_1^\top p_t$ is lower bounded by $q_1^\top p_0$ ( Proposition~\ref{prop:convergence}).

Putting it together for the case that $q_e=q_1$,
\[\E[R_e(T) |q_e=q_1] = \sum_{t=T_e}^{T-1} \Pr\left[q=q_2\right](q_1^\top \bar p_t - q_2^\top p_t) \leq  (T-T_e)2e^{-(p_0^\top q_1)^2T_e/\sigma^2}\]
where we make use of the fact that inner products are bounded by 1 for unit norm vectors.

Now we turn to the case that $q_e = q_2$ and break the regret into three terms
\[\E[R_e(T) |q_e=q_2] = \sum_{t=0}^{T_e-1} q_1^\top \bar p_t - q_2^\top p_t  +\sum_{t=T_e}^T \Pr[q=q_2](q_1^\top \bar p_t - q_2^\top p_t) + \Pr[q=q_1]q_1^\top (\bar p_t -  p_t)\]
The first can be bounded by $2T_e$
using that the vectors are unit norm.
For the second summation, a similar argument to the $q_e=q_1$ case shows that the probability of failure is bounded by $e^{-(p_0^\top q_2)^2T_e/\sigma^2}$.
The third term remains: even when failure does not occur, there is suboptimality due to the fact that $q_e\neq q_1$ results in differing trajectories $\bar p_t\neq p_t$.

Notice that even when $q_e=q_2$, it is still the case that $q_1^\top p_t >0$ for all $t$.
In other words, application of $q_2$ preserves the sign of the inner product $q_1^\top p_t$:
\[q_1^\top p_+ \propto q_1^\top p + \eta q_1^\top q_2 p^\top q_2> 0\]
where we use the assumption that $q_1^\top q_2<0$.
Therefore, we can bound the difference $q_1^\top\bar p_t -  q_1^\top p_t$ using Lemma~\ref{lem:convergence} and considering a time offset of $T_e$. 
\begin{align*}
    q_1^\top\bar p_t - q_1^\top p_t &\leq   \gamma_{t-T_e} \frac{ q_1^\top (\bar p_{T_e} - p_{T_e})}{(\bar p_{T_e}^\top q_1)^2(p_{T_e}^\top q_1)^2}\:. %
\end{align*}
We can lower bound $\bar p_{T_e}^\top q_1$ by $p_0^\top q_1$ using Proposition~\ref{prop:convergence}.
To lower bound $p_{T_e}^\top q_1$, we first use Corollary~\ref{lem:alpha_beta} to write that
\[p_{T_e}^\top q_1 \geq \min_{\alpha,\beta} \quad \alpha q_1^\top p_0 - \beta q_1^\top q_2\quad \text{s.t.}\quad \alpha^2+\beta^2+2\alpha\beta p_0^\top q_2 = 1\]
Lemma~\ref{lem:ellipse_constrained_opt} shows that the minimum of this expression is $\min\{q_1^\top p_0, -q_1^\top q_2\}$.

This allows us to bound the expression

Finally, recalling that the sum of $\gamma_t$ is bounded by ,
\begin{align*}
    \E[R_e(T) |q_e=q_2]
    &\leq 2T_e + (T-T_e)2e^{-(p_0^\top q_2)^2T_e/\sigma^2}+  \frac{\sum_{t=T_e}^T  \gamma_{t-T_e} q_1^\top (\bar p_{T_e} - p_{T_e})}{(p_{0}^\top q_1)^2\min\{q_1^\top p_0, -q_1^\top q_2\}^2}
\end{align*}
The numerator will be bounded by $C_\gamma$ as defined in Proposition~\ref{prop:fixed_regret}.

Putting it all together with $a=\min\{|p_0^\top q_1|, |p_0^\top q_2|, |q_1^\top q_2|\}$ 
\begin{align*}
    \E[R_e(T)]  &\leq \frac{1}{2} \E[R_e(T) |q_e=q_1] + \frac{1}{2} \E[R_e(T) |q_e=q_2]\\
    &\leq T_e + 2(T-T_e) e^{-a^2T_e/\sigma^2} + \frac{C_\gamma}{2a^4}
\end{align*}
Setting $T_e=\sigma^2 \log(T)/a^2$ results in 
\[\E[R_e(T)] \leq \sigma^2 \log(T)/a^2 + 2 + \frac{C_\gamma}{2a^4}\]
Then the final result holds using Proposition~\ref{prop:fixed_regret} to bound $R_{q_1}(T) $ along with the observation that since $a<1$,
\[\frac{1}{2a^4} + \frac{1}{a^{2}}- 1 \leq \frac{1}{a^4}\:.\]
\end{proof}

\section{Proofs of results in Section~\ref{sec:rand}}\label{app:main_rand_res}

\subsection{Randomized Convergence for Corollary~\ref{prop:random_conv}}
\begin{proof}[Proof of Corollary~\ref{prop:random_conv}]

First, notice that we can bound
\[1-(p_t^\top v_1 )^2 = p_t^\top(I- v_1v_1^\top) p_t = \|V_{2:}^\top p_t\|^2 \leq \|V_{2:}^\top p_t(v_1^\top p_t)^{-1}\|^2\:.\]
Then, by the fundamental theorem of linear algebra,
\[1=\|p_0\|^2 = (p_0^\top v_1)^2 + \|V_{:2}^\top p_0\|^2\:.\]
Therefore, 
\[v_1^\top p_0 \geq \sqrt{2}/2 \implies \|V_{:2}^\top p_0(v_1^\top p_0)^{-1}\|\leq 1 \:.\]

Therefore, we can apply Theorem 12 from~\citet{huang2021streaming} with the identification that step size parameters are
$\eta(\lambda_1-\lambda_2)$ and $s+1$.
Union bounding their result, we have that for $0\leq t\leq T-1$,
$$1-(p_t^\top v_1 )^2\leq \|V_{2:}^\top p_t(v_1^\top p_t)^{-1}\| \leq 2e\sqrt{\frac{s+2}{s+1+t}}\:.$$

\end{proof}

\subsection{Feasibility in Proposition~\ref{prop:eigen}}
\begin{proof}[Proof of Proposition~\ref{prop:eigen}]
First consider
$$\tilde {\mathcal C} = \left\{\alpha\in\R^N, \lambda\in\R \mid \alpha\geq 0, \sum_{i=1}^N \alpha_i = 1, Q\diag(Q^\top v)\alpha = \lambda v\right\}.$$
This set exactly describes all possible weightings and eigenvalues because
\[\E[qq^\top] v = \lambda v \iff Q (Q^\top v \circ \alpha) = \lambda v\:,\]
and the first two constraints are required so that $\alpha$ is a valid probability weighting.

We now show that $\mathcal C_\mathrm{eig}$ and $\tilde {\mathcal C}$ are equivalent.
First, notice that $x\in\mathcal C_\mathrm{eig}$ implies that the pair $(x/1^\top x,(1^\top x)^{-1})\in\tilde{\mathcal C}$.
Second, note that if $\alpha,\lambda \in\tilde{\mathcal C}$, then $\alpha/\lambda \in \mathcal C_\mathrm{eig}$, noting that $\lambda>0$ by assumption.

Finally, for $\mathcal C_{\mathrm{dom}}$ it suffices to use the identification $x=\alpha/\lambda$ and note that
\[I - Q\diag(x)Q^\top \succeq 0  \iff \lambda I  \succeq Q\diag(\alpha)Q^\top \iff \lambda \geq \lambda_{N}(Q\diag(\alpha)Q^\top)\]
Since the previous argument showed is that $\lambda$ is an eigenvector of $Q\diag(\alpha)Q^\top$, it must be that it is identically the maximum eigenvector $\lambda = \lambda_{N}$.
\end{proof}

\section{Lemmas}
\subsection{Lemmas about preference dynamics}

\begin{lemma}\label{lem:convergence}
Consider the dynamics of fixed recommendations $q_t=q$ for different initial preferences $p_0$ and $\bar p_0$ such that $q^\top p_0 \cdot q^\top \bar p_0 > 0$.
Then we have that
\[q^\top (\bar p_t - p_t) \leq \gamma_t^2 \frac{ q^\top (\bar  p_0 - p_0)}{(\bar p_0^\top q)^2(p_0^\top q)^2}\]
\end{lemma}
\begin{proof}
Assume without loss of generality that $q$ has positive inner product with $\bar p_0$ and $p_0$ and that $\bar p_0^\top q> p_0^\top q$. (Otherwise, the proof holds by replacing $q$ with $-q$ and switching the roles of $p_0$ and $\bar p_0$.)
Applying Proposition~\ref{prop:convergence},
\[p_t^\top q= \left(1+ \gamma_t^2  c_0 \right)^{-1/2}, \quad \bar p_t^\top q= \left(1+ \gamma_t^2  \bar c_0 \right)^{-1/2}\]
where we define  $c_0 = \left(( p_0^\top q)^{-2}- 1\right)$ and similar for $\bar c_0$, noting that $\bar c_0<c_0$ and thus $p_t^\top q <\bar p_t^\top q$.
The difference can be written as
\[\bar p_t^\top q -  p_t^\top q= \left(1+ \gamma_t^2  \bar c_0 \right)^{-1/2} - \left(1+ \gamma_t^2   c_0 \right)^{-1/2} = \frac{\left(1+ \gamma_t^2 \bar c_0 \right)^{-1} - \left(1+ \gamma_t^2  c_0 \right)^{-1}}{\left(1+ \gamma_t^2 c_0 \right)^{-1/2} +\left(1+ \gamma_t^2 \bar c_0 \right)^{-1/2}}\:.\]
We first lower bound the denominator using the fact that $\gamma_t$ is decreasing in $t$ and $\gamma_0=1$.
\[\left(1+ \gamma_t^2 c_0 \right)^{-1/2} +\left(1+ \gamma_t^2 \bar c_0 \right)^{-1/2} \geq \left(1+ c_0 \right)^{-1/2} +\left(1+  \bar c_0 \right)^{-1/2} = p_0^\top q + \bar p_0^\top q\]
Next, we upper bound the numerator using the fact that $\gamma_t\geq 0$.
\[\frac{1}{1+ \gamma_t^2 \bar c_0 } - \frac{1}{1+ \gamma_t^2  c_0 } =\frac{\gamma_t^2 ( c_0 -  \bar c_0)}{(1+ \gamma_t^2 c_0 )(1+ \gamma_t^2 \bar c_0 )} \leq \gamma_t^2 ((  p_0^\top q)^{-2} -  ( \bar p_0^\top q)^{-2})\]
Combining these expressions, 
\[\bar p_t^\top q -  p_t^\top q \leq \gamma_t^2\frac{(  p_0^\top q)^{-2} -  ( \bar p_0^\top q)^{-2}}{p_0^\top q + \bar p_0^\top q} = \gamma_t^2\frac{ \bar p_0^\top q -  p_0^\top q}{(\bar p_0^\top q)^2(p_0^\top q)^2}\:.\]
\end{proof}

\subsection{Technical Lemmas}

\begin{lemma}\label{lem:inc_dec}
For $a\in(0,1)$, consider the functions
\[f_1(x) = \frac{x}{\sqrt{a+x^2(1-a)}}, \quad f_2(x) = \frac{1-x}{\sqrt{a+x^2(1-a)}}, \quad f_3(x) = \frac{x+a^2(1-x)}{\sqrt{a+x^2(1-a)}}\:.\]
Then for $x\geq 0$, $f_1$ and $f_3$ are increasing in $x$ while $f_2$ is decreasing.
\end{lemma}
\begin{proof}
The functions increase (or decrease) if and only if their derivatives are positive (or negative).
The derivatives are
\[f'_1(x) =  \frac{ a }{d(x)} > 0, 
\quad
f'_2(x) = \frac{
-a-(1-a) x 
}{d(x)}< 0,\quad
f'_3(x) = \frac{
 a (1-a^2)  + a^2(1-a) x
}{d(x)}> 0\:.
\]
where $d(x) =(a+x^2(1-a))^{3/2}$.
\end{proof}

\begin{lemma}\label{lem:ellipse_constrained_opt} For $a,b,c\in(0,1]$,
\[\min\{a,b\} = \min_{x,y\in\R_+} \quad ax+by \quad \text{s.t.}\quad x^2+y^2-2cxy = 1\]
\end{lemma}
\begin{proof}
The constraint set of this optimization is the arc of an ellipse contained within the positive quadrant.
Therefore, the minimum must occur at a boundary (where $x=0$ or $y=0$) or at a critical point in the interior.
If the minimum occurs at a boundary, it will be $\min\{a,b\}$.

We now argue that the minimum cannot occur in the interior. 
Notice that
\[ax+by = \left\|\begin{bmatrix}a\\b\end{bmatrix}\right\|_2  \left\|\begin{bmatrix}x\\y\end{bmatrix}\right\|_2 \cos \theta =\left\|\begin{bmatrix}a\\b\end{bmatrix}\right\|_2 \sqrt{1+2cxy} \cos\theta \]
where $\theta$ is the angle between $[a,b]$ and $[x,y]$ and the equality follows by the ellipse constraint.
Both terms depending on $x,y$ attain their minimum on the boundary: this is where the norm is smallest and the angle is largest.
\end{proof}
 \end{document}